\documentclass[9pt]{article}

\usepackage{spconf}
\usepackage{graphicx,epstopdf,hyperref,subcaption}
\usepackage{enumitem}
\usepackage{algorithm,algorithmic}
\usepackage{multirow}
\usepackage{amsthm,amsmath,amssymb,braket}
\usepackage[font=footnotesize,skip=5pt]{caption}

\DeclareMathOperator*{\argmin}{arg\,min}
\DeclareMathOperator*{\argmax}{arg\,max}

\usepackage{bm}

\long\def\comment#1{}

\newfont{\bbb}{msbm10 scaled 700}

\newcommand{\av}{{\bf a}}

\newcommand{\fv}{{\bf f}}
\newcommand{\gv}{{\bf g}}

\newcommand{\uv}{{\bf u}}

\newcommand{\xv}{{\bf x}}

\newcommand{\zerov}{{\bf 0}}
\newcommand{\onev}{{\bf 1}}

\newcommand{\Am}{{\bf A}}

\newcommand{\Dm}{{\bf D}}

\newcommand{\Km}{{\bf K}}
\newcommand{\Lm}{{\bf L}}

\newcommand{\Sm}{{\bf S}}

\newcommand{\Um}{{\bf U}}
\newcommand{\Wm}{{\bf W}}

\newcommand{\Ec}{{\cal E}}

\newcommand{\Nc}{{\cal N}}

\newcommand{\Rc}{{\cal R}}
\newcommand{\Sc}{{\cal S}}

\newcommand{\Vc}{{\cal V}}

\newcommand{\Lcb}{{\bm {\mathcal L}}}

\newcommand{\muv}{\hbox{\boldmath$\mu$}}

\newcommand{\Lambdam}{\hbox{\boldmath$\Lambda$}}

\newcommand{\Sigmam}{\hbox{\boldmath$\Sigma$}}

\newcommand{\diag}{{\hbox{diag}}}

\newcommand{\trace}{{\hbox{tr}}}

\newcommand{\Scc}{{\Sc^c}}

\newtheorem{theorem}{Theorem}
\newtheorem{proposition}{Proposition}
\newtheorem{lemma}{Lemma}

\title{A Probabilistic Interpretation of Sampling Theory of Graph Signals}

\name{Akshay Gadde and Antonio Ortega
\thanks{This work was supported in part by NSF under grant CCF-1410009.}
}
\address{Department of Electrical Engineering\\
	University of Southern California, Los Angeles\\
	Email: agadde@usc.edu, ortega@sipi.usc.edu}

\begin{document}
\ninept

\maketitle
\begin{abstract}
We give a probabilistic interpretation of sampling theory of graph signals. To do this, we first define a generative model for the data using a pairwise Gaussian random field (GRF) which depends on the graph. We show that, under certain conditions, reconstructing a graph signal from a subset of its samples by least squares is equivalent to performing MAP inference on an approximation of this GRF which has a low rank covariance matrix. We then show that a sampling set of given size with the largest associated cut-off frequency, which is optimal from a sampling theoretic point of view, minimizes the worst case predictive covariance of the MAP estimate on the GRF. 
This interpretation also gives an intuitive explanation for the superior performance of  the sampling theoretic approach to active semi-supervised classification. 
\end{abstract}
\begin{keywords}
Graph Signal Processing, Sampling theorem, Gaussian Markov random field, Semi-supervised learning, Active learning
\end{keywords}
%
\vspace{-0.75\baselineskip}
\section{Introduction}
Graph signal processing aims to extend the tools for analysis, approximation, denoising and interpolation of traditional signals to signals defined on graphs. The advantage of this framework is that it allows us to process the given data while taking into consideration the underlying connectivity between the data points. The graph can be inherent to the data as is the case in application areas such as social networks and sensor networks or it can be constructed using the data to capture the underlying geometry. Examples of the latter are found in image processing and machine learning (see \cite{Shuman-SPM-13,Sandryhaila-SPM-14}).

In this paper, we focus on the sampling theory of graph signals. The classical Nyquist-Shannon sampling theorem says that a signal with bandwidth $f$ is uniquely determined by its (uniformly spaced) samples if the sampling rate is higher than $2f$. Intuitively, it tells us how ``smooth" the signal has to be, for perfect recovery, given the sampling density, and vice versa. Moreover, the signal can be perfectly reconstructed from the samples by a simple low pass filter. Sampling theory of graph signals similarly deals with the problem of reconstructing an unknown graph signal from its samples on a subset of nodes. Frequency domain representation of graph signals is given by the eigenvectors and eigenvalues of the Laplacian matrix associated with the graph. In order to pose a sampling theorem analogous to the Nyquist-Shannon sampling theorem, we need to find the maximum bandwidth (in the graph spectral domain) that a graph signal can have so that it is uniquely determined by its samples on the given subset of nodes. Conversely, given the bandwidth, we need to find the smallest subset of nodes, so that recovery of any signal with that bandwidth, from its samples on that subset, is unique and stable. Given that the signal is smooth enough to be uniquely represented by its samples on a subset of nodes, we need to give an efficient and stable algorithm to reconstruct the unknown samples. These questions have been answered to some extent in \cite{Pesenson-AMS-08,Narang-ICASSP-13,Narang-GlobalSIP-13,Anis-ICASSP-14}. We discuss some of these results in Section~\ref{sec:samp_th}.

This sampling theoretic perspective has been shown to be very useful for graph based active semi-supervised learning~\cite{Gadde-KDD-14}. 
In this context, label prediction is considered as a graph signal reconstruction problem. The characterization of a subset of nodes given by the sampling theory, namely the associated cutoff frequency
is used as a criterion function to choose the optimal set nodes to be labelled for active learning. 

Sampling theoretic approaches for active and semi-supervised learning \cite{Gadde-KDD-14} are purely deterministic. However, their probabilistic interpretation is desired for the 
following reasons: 1.~It allows us to understand them as model based methods and thus, makes it easier to include them as components of a larger probabilistic model. 
2.~It can also suggest a principled way to refine the model parameters (which are given by the underlying graph) as more data is observed (see~\cite{Kapoor-NIPS-05} for an example).
3.~The interpretation presented in this paper assumes a Gaussian random field model for the data. 
This may lead to generalizations of the sampling theory to data with non-Gaussian distributions which might be more realistic for a classification problem. 
4.~This interpretation also makes the relationship between the sampling theoretic approach and previously proposed semi-supervised \cite{Zhu-ICML-03} and active learning \cite{Ji-AIST-12, Ma-NIPS-13} methods more apparent as discussed in Section~\ref{sec:related_work}. 

The main contributions of this paper are the following. We define a generative model for graph signals using a pairwise Gaussian random field (GRF) with a covariance matrix that depends on the graph. We show that, when conditions of the graph signal sampling theorem are satisfied, bandlimited reconstruction of a graph signal from a subset of its samples is equivalent to performing MAP inference on a low rank approximation of the above GRF. 
This learning model performs very well in classification problems, as demonstrated in the experiments, since the true data covariance matrix is expected to be close to low rank. We then show that a sampling set of given size with the largest associated cut-off frequency, which is optimal from a sampling theoretic point of view, minimizes the worst case predictive covariance of the MAP estimate on the GRF.


\vspace{-0.5\baselineskip}
\section{Sampling Theory of Graph Signals} 
\subsection{Preliminaries and Notation}
We consider a connected, undirected and weighted graph $G = (\Vc,\Ec)$. The nodes $\Vc$ in the graph are indexed by $\{1,2,\ldots,N\}$. $\Scc$ denotes the complement of $\Sc$ in $\Vc$, i.e., $\Scc = \Vc \setminus \Sc$. The edge set $\Ec$ is given by $\{(i,j,w_{ij})\}$, where $i,j \in \Vc$ and $w_{ij} \in \mathbb{R}^+$. $(i,j,w_{ij})$ denotes an edge with weight $w_{ij}$ connecting nodes $i$ and $j$. The connectivity information given by $\Ec$ is encoded by the adjacency matrix $\Wm$ of size $N \times N$ with $\Wm(i,j) = w_{ij}$. The degree matrix $\Dm$ is a diagonal matrix $\diag\{d_1, \ldots d_N\}$, where $d_i = \sum_j w_{ij}$ is the degree of node $i$. The Laplacian matrix is defined as $\Lm = \Dm - \Wm$. The symmetric normalized form of the Laplacian is given by $\Lcb = \Dm^{-1/2} \Lm \Dm^{-1/2}$.
A graph signal $f: \Vc \rightarrow \mathbb{R}$ is a mapping which takes a real value on each node of the graph. It can be represented as $\fv = (\fv_1, \ldots, \fv_N)^\top \in \mathbb{R}^N$. 
For $\xv \in \mathbb{R}^N$, $\xv_\Sc$ denotes a sub-vector of $\xv$ consisting of its components indexed by $\Sc$. Similarly, for $\Am \in \mathbb{R}^{N\times N}$, $\Am_{\Sc_1\Sc_2}$ is the sub-matrix of $\Am$ with rows indexed by $\Sc_1$ and columns indexed by $\Sc_2$. For simplicity, we denote $\Am_{\Sc\Sc}$ by $\Am_\Sc$. We use $\lambda_\text{max}[.]$ and $\lambda_\text{min}[.]$ to denote the largest and the smallest eigenvalue of a matrix, respectively. $\trace(.)$ denotes the trace of a matrix. $\Am^+$ is used to denote the pseudo-inverse of $\Am$. $\onev$ and $\zerov$ denote vectors or matrices of ones and zeros, respectively. 

It can be shown that $\Lm$ and $\Lcb$ are positive semi-definite. Hence, $\Lm$ has real eigenvalues $0 = \lambda_1 < \lambda_2 \leq \ldots \leq \lambda_N$ and a corresponding orthogonal set of eigenvectors $\{\uv^1, \uv^2, \ldots, \uv^N\}$. It can be diagonalized as $\Lm = \Um \Lambdam \Um^\top$,
where $\Um = (\uv^1, \ldots, \uv^N)$ and $\Lambdam = \diag\{\lambda_1, \ldots, \lambda_N\}$.
Variation in the eigenvectors of $\Lm$ over the graph (as captured by $\uv^\top \Lm \uv = \sum_{i,j} w_{ij}(\uv_i - \uv_j)^2$) increases as the corresponding eigenvalues increase. Thus, these eigenvectors allow us a to define a graph dependent notion of frequency for the graph signals. The so-called Graph Fourier Transform (GFT)\footnote{The GFT is usually defined using the normalized Laplacian $\Lcb$. We define it using $\Lm$ for the sake of notational simplicity. However, most of the discussion in the paper can be easily generalized to $\Lcb$.} is defined as $\tilde{\fv}_i = \Braket{\fv, \uv^i}$ (or in an equivalent matrix form $\tilde{\fv} = \Um^\top \fv$), where $\tilde{\fv}_i$ is the GFT coefficient corresponding to frequency $\lambda_i$. An $\omega$-bandlimited signal has its GFT supported on $[0,\omega]$, i.e., $\tilde{\fv_i} = 0$ for $\lambda_i > \omega$. Conversely, such a signal is said to have a bandwidth equal to $\omega$. If $\{\lambda_1, \ldots, \lambda_r\}$ are the eigenvalues less than $\omega$, then any $\omega$-bandlimited signal can be written as a linear combination of corresponding eigenvectors
%
\begin{equation}
\fv = \sum_{i=1}^r \av_i\uv^i = \Um_{\Vc\Rc} \av,
\label{eq:bl_sig}
\end{equation}
where $\av$ is the coefficient vector. The space of $\omega$-bandlimited signals is called a Paley-Wiener space $PW_\omega(G)$.
%

\vspace{-0.5\baselineskip}
\subsection{Sampling Theorem and Bandlimited Reconstruction}
\label{sec:samp_th}
Sampling theory deals with the problem reconstructing an $\omega$-bandlimited signal $\fv$ from its samples $\fv_\Sc$ on the nodes in $\Sc \subseteq \Vc$. There are three important questions that need to be answered in this context:
1.~Given $\Sc$, what is the maximum bandwidth $\omega$ that $\fv$ can have so that it is uniquely determined by $\fv_\Sc$? 
2.~Which is the best sampling set $\Sc_\text{opt}$ of a given size $m$?  
3.~Given that $\fv$ is uniquely determined by $\fv_\Sc$, how to find the unknown samples $\fv_\Scc$? 
We briefly review some of the results related to each of the above problems.

Let $L_2(\Scc)$ be the space of signals which are identically zero on $\Sc$ but can have non-zero samples on $\Scc$, i.e., $\gv_\Sc = \zerov\;\forall\gv \in L_2(\Scc)$. It is easy to see that for all signals in $PW_\omega(G)$ to be uniquely determined by their samples on $\Sc$, we need $PW_\omega(G) \cap L_2(\Scc) = \{\zerov\}$. This observation leads to the following theorem. 
\begin{theorem}[Sampling Theorem~\cite{Anis-ICASSP-14}] 
Any signal in $PW_\omega(G)$ can be uniquely reconstructed from its samples on a subset of nodes $\Sc$ if and only if
%
\begin{equation}
\omega < \inf_{\gv \in L_2(\Scc)} \omega(\gv),
\end{equation}
where $\omega(.)$ denotes the bandwidth of a signal. If the above condition is satisfied, then $\Sc$ is said to be a uniqueness set for $PW_\omega(G)$.
\end{theorem}
To ensure unique recovery of a signal from its samples on $\Sc$, its bandwidth has to be less than $\inf_{\gv \in L_2(\Scc)} \omega(\gv)$. This is called the cut-off frequency associated with the subset $\Sc$ and is denoted by $\omega(\Sc)$. 
An estimate of the cut-off frequency is given by~\cite{Anis-ICASSP-14}
\begin{equation}
\Omega_k(\Sc) = \left(\lambda_\text{min}\left[(\Lm^k)_\Scc \right] \right)^{1/k}.
\end{equation}
It can be shown that $\Omega_k(\Sc) \leq \omega(\Sc)$ and we get closer to $\omega(\Sc)$ as $k$ increases.  

A larger cut-off frequency estimate $\Omega_k(\Sc)$ implies that a bigger space of signals can be perfectly recovered from their samples on $\Sc$. Therefore, $\Omega_k(\Sc)$ can be used as a criterion function to be maximized for choosing the optimal sampling set $\Sc_\text{opt}$ of given size $m$, i.e., 
\begin{equation}
\Sc_\text{opt} = \argmax_{|\Sc| = m} \Omega_k(\Sc).
\label{eq:set_select}
\end{equation}
The above problem is combinatorial and NP-hard. A greedy algorithm for finding an approximate solution is proposed in~\cite{Anis-ICASSP-14}.

Consider a signal $\fv \in PW_\omega(G)$ with $\omega < \omega(\Sc)$. Using the representation of a bandlimited signal in \eqref{eq:bl_sig}, we get that $\fv_\Sc = \Um_{\Sc\Rc}\av$. Since $\fv$ is uniquely sampled on $\Sc$, $\Um_{\Sc\Rc}$ must have full column rank so that the least squares solution $\av$ of the above system of equations is unique. The unknown samples can then be reconstructed by:
\begin{equation}
\fv_\Scc = \Um_{\Scc\Rc}(\Um_{\Sc\Rc}^\top\Um_{\Sc\Rc})^{-1}\Um_{\Sc\Rc}^\top \fv_\Sc.
\label{eq:bl_recon}
\end{equation}
A faster, iterative method for bandlimited reconstruction is proposed in~\cite{Narang-GlobalSIP-13}, which does not need the computation of eigenvectors.  

These sampling theory based algorithms for subset selection and signal reconstruction have been applied to graph based active semi-supervised learning and are shown to perform better than many state of the art approaches~\cite{Gadde-KDD-14}.

%
\vspace{-0.25\baselineskip}
\section{GRF Model for Graph Signals}

%
%
In order to give a probabilistic interpretation of the graph signal processing framework, we define a generative model for the signal using a pairwise Gaussian Random Field (GRF) based on the graph $G$. A random signal $\fv = (\fv_1, \ldots, \fv_N)^\top$ is assumed to be drawn from the following distribution:
\begin{align}
p(\fv) &\propto \exp \left( - \sum_{i,j} w_{ij}(\fv_i - \fv_j)^2 -  \delta \sum_i \fv_i^2 \right)  \nonumber \\
& = \exp \left(-\fv^\top(\Lm + \delta \mathbf{I}) \fv \right),
\label{eq:grf}
\end{align}
where $\mathbf{I}$ denotes an identity matrix of size $N \times N$. Let $\Km$ be the covariance matrix of the the GRF. Then, from the above equation, the inverse covariance matrix (also known as the precision matrix) can be written as:
\begin{equation}
\Km^{-1} = \Lm + \delta \mathbf{I}.
\end{equation} 
Note that $\Km$ has the same eigenvectors as $\Lm$, while the corresponding eigenvalues are $\sigma_i = \frac{1}{\lambda_i + \delta}$. Thus, $\Km$ can be diagonalized as
\begin{equation}
\Km = \sum_{i=1}^N\frac{1}{\lambda_i + \delta}\uv^i {\uv^i}^\top = \Um \Sigmam \Um^\top,
\label{eq:K_eig}
\end{equation} 
where $\Sigmam = \diag\{\sigma_1,\ldots,\sigma_N\}$. The advantage of introducing the parameter $\delta$ is that it leads to a non-singular precision matrix and thus, allows us to have a proper covariance matrix. $\sigma_1 = 1/\delta$ can be thought of as the variance of the DC component of $\fv$ since $\uv^1 = \onev$.



\section{Sampling Theory and Inference over GRF}
Consider a signal $\fv$ generated using the GRF defined in \eqref{eq:grf} with covariance matrix $\Km = (\Lm + \delta\mathbf{I})^{-1}$. As in the sampling problem, we observe the samples of $\fv$ on a subset $\Sc$ of nodes. Our goal is to estimate the unknown samples. 
It is well known that the conditional distribution of $\fv_\Scc$ given $\fv_\Sc$ equals $\Nc(\muv_{\Scc|\Sc}, \Km_{\Scc|\Sc})$, where
\begin{align}
\muv_{\Scc|\Sc} &= \Km_{\Scc\Sc} (\Km_\Sc)^{+} \fv_\Sc \text{ and} \\ 
\Km_{\Scc|\Sc} &= \Km_{\Scc} - \Km_{\Scc\Sc} (\Km_\Sc)^{+} \Km_{\Sc\Scc}
\label{eq:map_est}
\end{align}
are the MAP estimate and the predictive covariance matrix of $\fv_\Scc$ given $\fv_\Sc$, respectively~\cite{Zhu-ICML-03,Scholtz-562a}.

\vspace{-0.5\baselineskip}
\subsection{Bandlimited Reconstruction as MAP Inference}
Let $\lambda_r$ be the largest eigenvalue of $\Lm$ which is less than $\omega$. We define
$\hat{\Km}$ to be a low rank approximation of $\Km$ which only contains the spectral components corresponding to $\{\lambda_1, \ldots, \lambda_r\}$, i.e.,
\vspace{-0.25\baselineskip}
\begin{equation}
\hat{\Km} = \sum_{i = 1}^r \frac{1}{\lambda_i + \delta} \uv^i {\uv^i}^\top = \Um_{\Vc\Rc} \Sigmam_\Rc \Um_{\Vc\Rc}^\top.
\label{eq:K_low_rank}
\end{equation}
%
Consider the problem of reconstructing a random signal generated using a GRF with covariance $\hat{\Km}$, from its samples on $\Sc$. The following theorem shows that, if conditions of the sampling theorem are satisfied, then the error of bandlimited reconstruction is zero.
\vspace{-0.25\baselineskip}
\begin{theorem}
Let $\fv$ be a random graph signal generated using the GRF with covariance $\hat{\Km}$ given by \eqref{eq:K_low_rank}. Let $\hat{\fv}_\Scc$ be the bandlimited reconstruction of $\fv_\Scc$ obtained from its samples on $\Sc$, where $\Sc$ is a uniqueness set for $PW_\omega(G)$. Then, $\|\fv_\Scc-\hat{\fv}_\Scc\| = 0$.
\label{th:zero_avg_err}
\end{theorem}
\vspace{-0.25\baselineskip}
Before proving the above theorem, we show, in the following lemma, that bandlimited reconstruction is equivalent to MAP inference on the GRF with covariance $\hat{\Km}$. 
\vspace{-0.0\baselineskip}
\begin{lemma}
\label{th:bl_recon_low_rank}
Let $\Sc \subseteq \Vc$ be a uniqueness set for $PW_\omega(G)$. 
Then the MAP estimate of $\fv_{\Scc}$ given $\fv_{\Sc}$ in a GRF with covariance matrix $\hat{\Km}$ is equal to the bandlimited reconstruction given by~\eqref{eq:bl_recon}. 
\end{lemma}
\vspace{-0.95\baselineskip}
\begin{proof}
Under a permutation which groups together nodes in $\Scc$ and $\Sc$, we can write $\hat{\Km}$ as the following block matrix
\begin{align}
\begin{bmatrix}
\hat{\Km}_\Scc & \hat{\Km}_{\Scc\Sc} \\
\hat{\Km}_{\Sc\Scc} & \hat{\Km}_\Sc
\end{bmatrix}
= \begin{bmatrix}
\Um_{\Scc\Rc} \Sigmam_\Rc \Um_{\Scc\Rc}^\top & \Um_{\Scc\Rc} \Sigmam_\Rc \Um_{\Sc\Rc}^\top \\
\Um_{\Sc\Rc} \Sigmam_\Rc \Um_{\Scc\Rc}^\top & \Um_{\Sc\Rc} \Sigmam_\Rc \Um_{\Sc\Rc}^\top
\end{bmatrix}
\label{eq:blk_K}
\end{align}
Therefore, we can write the MAP estimate obtained with covariance $\hat{\Km}$ as, 
\begin{equation}
\hat{\muv}_{\Scc|\Sc} = \Um_{\Scc\Rc} \Sigmam_\Rc \Um_{\Sc\Rc}^\top (\Um_{\Sc\Rc} \Sigmam_\Rc \Um_{\Sc\Rc}^\top)^+ \fv_\Sc.
\label{eq:map_bl}
\end{equation}
Because $\omega < \omega(\Sc)$, we have that $\Um_{\Sc\Rc}$ has full column rank and equivalently, $\Um_{\Sc\Rc}^\top$ has full row rank. Therefore, we can write $(\Um_{\Sc\Rc} \Sigmam_\Rc \Um_{\Sc\Rc}^\top)^+ = (\Um_{\Sc\Rc}^\top)^+ \Sigmam_\Rc^+ \Um_{\Sc\Rc}^+$ and $\Um_{\Sc\Rc}^+ = (\Um_{\Sc\Rc}^\top\Um_{\Sc\Rc})^{-1} \Um_{\Sc\Rc}^\top$. Simplifying \eqref{eq:map_bl} using these equalities leads to
\begin{align*}
\hat{\fv}_\Scc 
&= \Um_{\Scc\Rc} (\Um_{\Sc\Rc}^\top\Um_{\Sc\Rc})^{-1} \Um_{\Sc\Rc}^\top \fv_\Sc,
\end{align*}
which is equal to the least squares solution given in \eqref{eq:bl_recon}.
\end{proof}
\begin{proof}[Proof of Theorem~\ref{th:zero_avg_err}]
From Lemma~\ref{th:bl_recon_low_rank}, $\hat{\fv}_\Scc = \hat{\muv}_{\Scc|\Sc}$. Therefore, $\mathbb{E}(\|\fv_\Scc-\hat{\fv}_\Scc\|^2) = \trace(\mathbb{E}(\fv_\Scc-\hat{\muv}_{\Scc|\Sc})(\fv_\Scc-\hat{\muv}_{\Scc|\Sc})^\top) = \trace(\hat{\Km}_{\Scc|\Sc})$. Now, $\hat{\Km}_{\Scc|\Sc} = \hat{\Km}_\Scc - \hat{\Km}_{\Scc\Sc}(\hat{\Km}_\Sc)^+\hat{\Km}_{\Sc\Scc}$. Using the block form of $\hat{\Km}$ in \eqref{eq:blk_K}, and the fact that $\Um_{\Sc\Rc}$ has full column rank, it is easy to show that $\hat{\Km}_{\Scc|\Sc} = \zerov$, which implies $\mathbb{E}(\|\fv_\Scc-\hat{\fv}_\Scc\|^2) = 0$. But since, $\|\fv_\Scc-\hat{\fv}_\Scc\| \geq 0,$ we get $\|\fv_\Scc-\hat{\fv}_\Scc\| = 0$. 
\end{proof}
%


\vspace{-0.75\baselineskip}
\subsection{Cut-off Frequency and Estimation Error}
If the true covariance matrix is only approximately low rank, then MAP inference with $\hat{\Km}$ gives a non-zero reconstruction error. The best sampling set in this case is the one which minimizes the predictive covariance.
According to the sampling theory of graph signals, the optimal sampling set of given size is the one which has the largest associated cut-off frequency. We show that finding a sampling set $\Sc$ which maximizes a crude estimate of the cut-off frequency $\Omega_1(\Sc)$ is equivalent to minimizing the maximum eigenvalue of the predictive covariance of ${\fv}_\Scc$ given $\fv_\Sc$. 
%
\begin{proposition}
Let $\Sc_\text{opt} = \argmax_{|\Sc| = m} \Omega_1(\Sc)$. Let $\Km =$ \hbox{$(\Lm + \delta \mathbf{I})^{-1}$.} Then, $\Sc_\text{opt} = \argmin_{|\Sc| = m} \lambda_\text{max}[\Km_{\Scc|\Sc}]$. 
\end{proposition}
\begin{proof}
Consider a block matrix representation of $\Km$ similar to \eqref{eq:blk_K}. Using the block matrix inversion formula, we can write $\Km^{-1}$ as
\begin{equation*}
\Km^{-1} = \begin{bmatrix}
\Sm_{\Km_\Sc}^{-1} & -(\Km_{\Scc})^{-1} \Km_{\Scc\Sc} \Sm_{\Km_\Scc}^{-1}\\
-(\Km_{\Sc})^{-1} \Km_{\Scc\Sc}^\top \Sm_{\Km_\Sc}^{-1} & \Sm_{\Km_\Scc}^{-1}
\end{bmatrix}, 
\end{equation*}
\vspace{-\baselineskip}
\begin{align}
\text{where} \quad
&\Sm_{\Km_\Sc} = \Km_\Scc - \Km_{\Scc\Sc} (\Km_\Sc)^{-1} \Km_{\Scc\Sc}^\top, \nonumber  \\
&\Sm_{\Km_\Scc} = \Km_\Sc - \Km_{\Scc\Sc}^\top (\Km_\Scc)^{-1} \Km_{\Scc\Sc}
\end{align}
are the Schur complements of $\Km_{\Sc}$ and $\Km_{\Scc}$ respectively. $\Lm_\Scc = (\Km^{-1})_\Scc - \delta \mathbf{I}_\Scc = \Sm_{\Km_\Sc}^{-1} - \delta \mathbf{I}_\Scc$. Note that $\Sm_{\Km_\Sc} = \Km_{\Scc|\Sc}$. Thus, the estimated cut-off frequency corresponding to the subset $\Sc$ of nodes can be written in terms of the conditional covariance matrix
\begin{equation}
\Omega_1(\Sc) = \lambda_\text{min}[\Lm_\Scc] = \frac{1}{\lambda_{\text{max}}[\Km_{\Scc|\Sc}]} - \delta.
\end{equation}
The result readily follows from this. 
\end{proof}
\vspace{-0.35\baselineskip}
A sampling set with the largest estimated cut-off frequency $\Omega_1(\Sc)$ also minimizes the worst case prediction error of the MAP estimate on a GRF with $\Km = (\Lm+\delta\mathbf{I})^{-1}$. However, as shown in Lemma~\ref{th:bl_recon_low_rank}, bandlimited signal reconstruction is equivalent to MAP estimation with a low rank approximation of $\Km$. Intuitively, a better estimate of the predictive covariance, in this model of signal reconstruction, can be obtained with by $((\Km^k)_{\Scc|\Sc})^{1/k}$ with larger values of $k$ as it gives more weight to the principal components with larger variance. This justifies the use of $\Omega_k(\Sc)$ with $k > 1$ as a criterion for active learning.  

%

\vspace{-0.5\baselineskip}
\subsection{Justification for the Sampling Theoretic Approach to Active Semi-supervised Classification}
\label{sec:justification}
MAP estimation is optimal for reconstructing signals generated using a GRF with a full rank covariance matrix, because it minimizes the mean squared error of estimation.  Moreover, since the estimation error equals $\trace(\Km_{\Scc|\Sc})$, an optimal sampling set of size $m$ is given by $\argmin_{|\Sc|=m} \trace(\Km_{\Scc|\Sc})$. Indeed, this is the so-called \emph{$V$-optimality criterion} for active learning proposed in \cite{Ji-AIST-12}.

However, in a classification problem, data points in the same class are highly correlated whereas data points in different classes have very small correlation. Since the number of classes is typically very small compared to the number of data points, we expect the (unknown) ``true" covariance matrix to be very well-approximated by a low rank matrix~\cite{Kuang-SDM-12}. Thus, bandlimited interpolation is a better model for signal reconstruction in this context, since it is equivalent to MAP estimation with a low rank covariance matrix. Maximizing the cut-off frequency is a natural set selection criterion for this learning model. 

\vspace{-0.75\baselineskip}
\section{Related Work}
\label{sec:related_work}
\vspace{-0.5\baselineskip}
Different criteria have been proposed for batch mode active learning on Gaussian random fields. The approach presented in \cite{Krause-JMLR-08} selects the points to label such the mutual information between the labelled and unlabelled data points is maximized. Our sampling theoretic approach \eqref{eq:set_select} is more similar to the methods proposed in \cite{Ji-AIST-12,Ma-NIPS-13}. These methods use MAP estimation on GRF \cite{Zhu-ICML-03} as their model for label prediction. 
As stated before, \cite{Ji-AIST-12} chooses the sampling set $\Sc$ by minimizing $\trace(\Km_{\Scc|\Sc})$. 
The method in \cite{Ma-NIPS-13}, on the other hand, tries to minimize $\sum_{ij}(\Km_{\Scc|\Sc})_{ij}$ (also known as \emph{$\Sigma$-optimality criterion}). This is equivalent to minimizing the risk of the surveying problem~\cite{Garnett-ICML-12}  (which is the problem of determining the proportion of nodes belonging to one class). All the above methods are closely related to the optimal design of experiments~\cite{Boyd-CVX}. Experiment design deals with the problem of estimating a vector from a set of linear measurements. The goal is to choose the optimal set of $m$ measurements so that the estimation error is minimized. Different error measures lead to different optimality criteria. For example, minimizing the trace of estimation covariance leads to $A$-optimal design whereas minimizing its determinant gives the $D$-optimal design. The sampling theoretic approach is closer to the so-called $E$-optimal design which minimizes the worst case prediction error given by the maximum eigenvalue of the predictive covariance matrix.



\vspace{-0.5\baselineskip}
\section{Experiments}
\vspace{-0.5\baselineskip}
To demonstrate the effectiveness of the framework of sampling theory, we first apply it to the problem of graph based active semi-supervised classification. 
%
In our experiment, we use a subset of the USPS handwritten digit dataset containing $100$ $16\times 16$ images each of digits 0 to 9. We construct a weighted $K$-NN graph of 1000 nodes with $K=10$ and the similarities given by $w_{ij} = \exp\left(-\frac{\|\xv^i - \xv^j\|^2}{\sigma^2}\right)$. The problem is to choose the nodes to be labelled and then predict the unknown labels from the queried labels. 
We consider different combinations of active learning criteria and learning models. As expected from the discussion in Section~\ref{sec:justification}, selecting the sampling set by maximizing the cutoff frequency and then performing bandlimited reconstruction outperforms $\Sigma$ and $V$-optimality criteria used in conjunction with MAP estimation (see Figure~\ref{fig:result_1}(a)). 
Even if the learning model is fixed to bandlimited interpolation, the sampling theoretic approach gives better results as seen in Figure~\ref{fig:result_1}(b)). This is because maximizing the cutoff frequency is a more suitable set selection criterion under this model.
%
\begin{figure}[t]
\centering
        \begin{subfigure}[b]{0.385\textwidth}
                \includegraphics[width=1\textwidth]{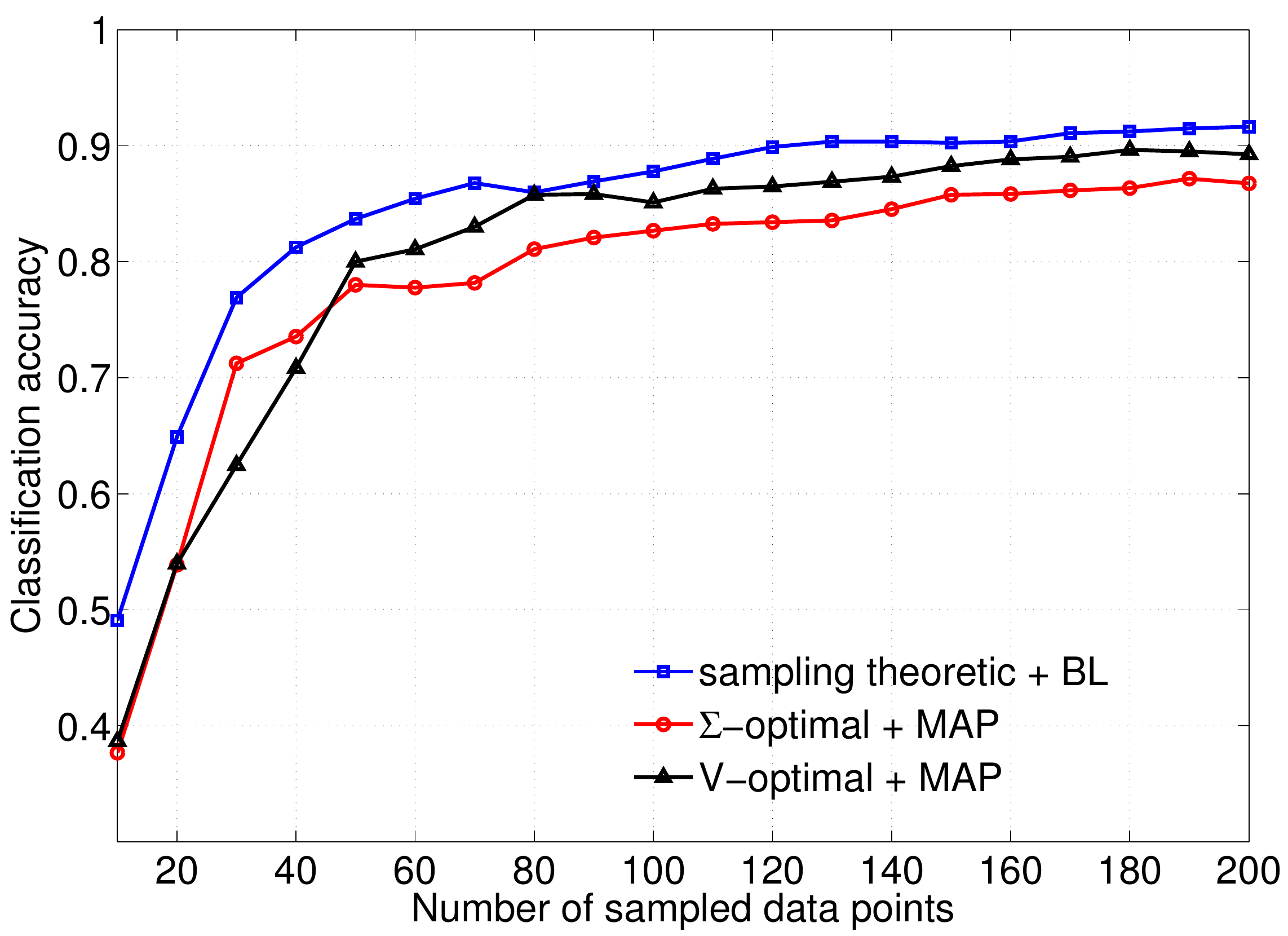}
                \caption{}
         \end{subfigure}
         \qquad
        \begin{subfigure}[b]{0.385\textwidth}
                \includegraphics[width=1\textwidth]{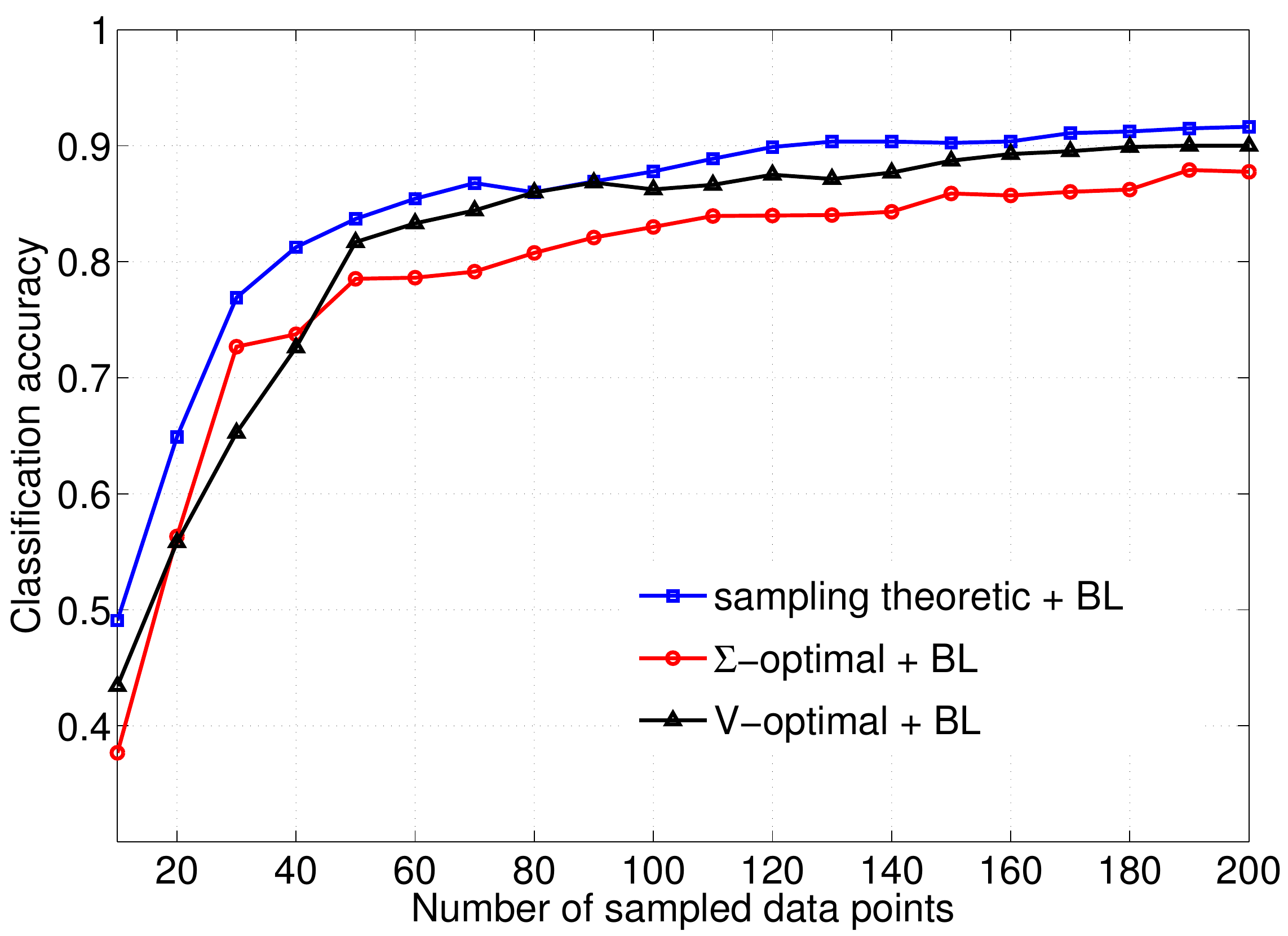}
                \caption{}
        \end{subfigure}
\caption{Figure shows the performance of different active learning criteria in conjunction with two learning models, namely, (a)~MAP~\cite{Zhu-ICML-03} and (b)~bandlimited reconstruction (BL)}
\vspace{-0.75\baselineskip}
\label{fig:result_1}
\end{figure}

On the other hand, if we consider the problem of regression of  a random real valued graph signal generated using a covariance matrix that is not low rank, 
a $V$-optimal set is expected to give a better SNR of reconstruction. This is demonstrated in Figure~\ref{fig:result_2} where we reconstruct a random real valued signal generated with the covariance matrix obtained using the graph from the previous example. 
%
\begin{figure}[t!]
  \centering
    \includegraphics[width=0.385\textwidth]{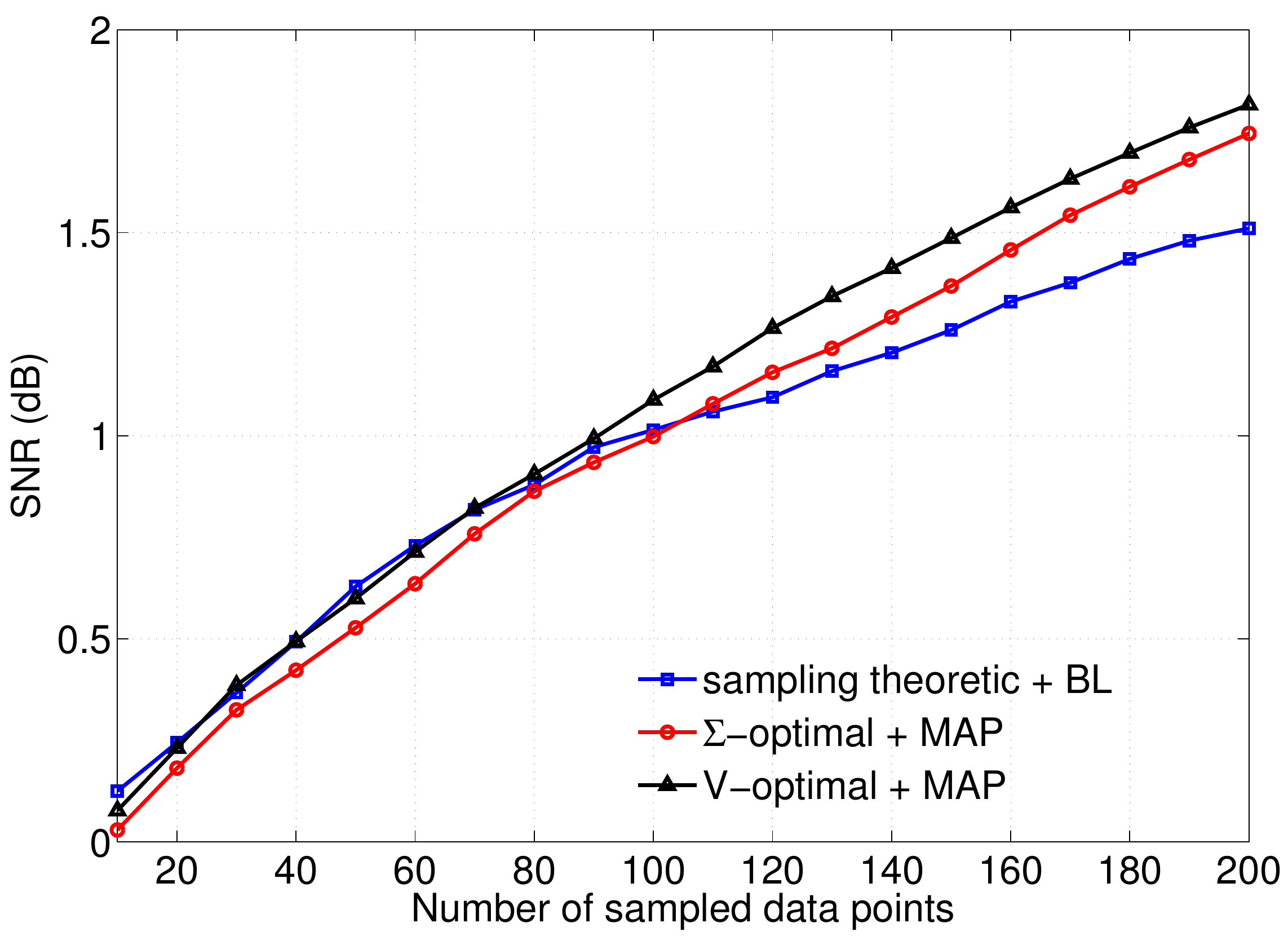}
    \caption{Performance in the case of reconstruction of a random real valued signal (averaged over 100 trials)}
    \vspace{-1.25\baselineskip}
    \label{fig:result_2}
\end{figure}

\vspace{-0.25\baselineskip}
\section{Conclusion and Future Work}
\vspace{-0.25\baselineskip}
In this paper, we gave a probabilistic interpretation for the sampling theory of graph signals. We showed that if the data is generated using a Gaussian random field whose precision matrix equals the graph Laplacian, then bandlimited reconstruction is equivalent to the MAP inference on an approximation of this GRF which has a low rank covariance matrix. Moreover, an optimal sampling set obtained via sampling theory minimizes the worst case predictive covariance of MAP estimation on the GRF. 

A probabilistic interpretation allows us to view graph signal sampling theory as a model based method. It would be interesting to consider it as part of a larger probabilistic model which refines the covariance matrix as more data is observed. This interpretation also suggests a generalization of the sampling theory to non-Gaussian models which might be more realistic for some applications. 
 
\bibliographystyle{IEEEbib}
\bibliography{refs}

\end{document}